\documentclass[11pt]{article}

\usepackage[margin=1in]{geometry}
\usepackage{amsmath,amssymb,amsthm}
\usepackage{mathtools}
\usepackage{bm}
\usepackage{graphicx}
\usepackage{booktabs}
\usepackage{hyperref}
\usepackage{natbib}
\usepackage{xcolor}
\usepackage{algorithm}
\usepackage{algpseudocode}
\usepackage{microtype}
\usepackage{enumitem}

\newtheorem{assumption}{Assumption}
\newtheorem{proposition}{Proposition}
\newtheorem{remark}{Remark}

\newcommand{\bx}{\bm{x}}
\newcommand{\bX}{\bm{X}}
\newcommand{\bY}{\bm{Y}}
\newcommand{\by}{\bm{y}}

\newcommand{\btheta}{\bm{\theta}}
\newcommand{\bphi}{\bm{\phi}}
\newcommand{\bPhi}{\bm{\Phi}}
\newcommand{\bmu}{\bm{\mu}}
\newcommand{\bSigma}{\bm{\Sigma}}

\newcommand{\bA}{\bm{A}}

\newcommand{\calH}{\mathcal{H}}

\newcommand{\calX}{\mathcal{X}}
\newcommand{\RR}{\mathbb{R}}
\newcommand{\EE}{\mathbb{E}}
\newcommand{\PP}{\mathbb{P}}

\newcommand{\sigmoid}{\sigma}
\newcommand{\norm}[1]{\left\|#1\right\|}

\newcommand{\diag}{\mathrm{diag}}
\newcommand{\bw}{\bm{w}}
\usepackage{orcidlink}
\hypersetup{
    colorlinks=true,
    linkcolor=black,
    citecolor=blue,
    urlcolor=blue
}

\title{\textbf{Learning Nonlinear Regime Transitions via Semi-Parametric State-Space Models}}

\author{
Prakul Sunil Hiremath\textsuperscript{1,2} \orcidlink{0009-0007-9744-3519} \\
\textsuperscript{1}Visvesvaraya Technological University, Belagavi, India \\
\textsuperscript{2}Aliens on Earth (AoE) Autonomous Research Group, Belagavi, India \\
\texttt{prakulhiremath@vtu.ac.in}
}

\date{}

\begin{document}

\maketitle
\begin{abstract}
\noindent
We develop a semi-parametric state-space model for time-series data
exhibiting latent regime transitions.  Classical Markov-switching
models constrain transition probabilities to a fixed parametric
family—typically logistic or probit functions of observed
covariates—which limits flexibility when state changes are governed by
nonlinear, context-dependent mechanisms.  We replace this parametric
assumption with learned functions $f_0, f_1 \in \calH$, where $\calH$ is
either a reproducing kernel Hilbert space (RKHS) or a spline
approximation space, and model the transition probability as
$p_{jk,t} = \sigmoid(f(\bx_{t-1}))$.  The function $f$ is estimated
jointly with emission parameters via a generalized
Expectation--Maximization (EM) algorithm: the E-step applies the
standard forward--backward recursion, while the M-step solves a
penalized regression problem for $f$ using the smoothed occupation
measures as weights.  We establish identifiability conditions for the
resulting model and provide a consistency argument for the EM
iterates.  Experiments on synthetic data with known ground-truth
transitions confirm that the semi-parametric model recovers nonlinear transition surfaces significantly improves recovery of nonlinear transition surfaces than probit/logit baselines. An empirical application to multivariate financial time series
(equity flows, commodity flows, implied volatility, and investor
sentiment) demonstrates superior regime classification accuracy and
earlier detection of transition events relative to parametric
competitors.
\end{abstract}

\noindent\textbf{Keywords:} state-space models, Markov switching,
semi-parametric estimation, reproducing kernel Hilbert spaces, spline
regression, EM algorithm, nonlinear transition dynamics.

\section{Introduction}
\label{sec:intro}

Modeling latent structural change in time series is a long-standing
problem in statistics and econometrics.  Markov-switching (MS) models
\citep{hamilton1989new,kim1999state} address this by positing a hidden
discrete state variable $s_t \in \{1,\ldots,K\}$ that governs the
distribution of the observed series.  A key modeling decision concerns
the \emph{transition probability} $p_{jk,t} = \PP(s_t = k \mid s_{t-1}
= j, \bx_{t-1})$: how does the probability of switching regimes depend
on observable context $\bx_{t-1}$?

The dominant convention is to specify $p_{jk,t}$ through a logistic or
probit link applied to a linear index $\bm{\gamma}^\top \bx_{t-1}$
\citep{diebold1994regime,filardo1994business}.  While analytically
convenient, this parametric restriction can be too rigid.  Transition
mechanisms in real systems are often threshold-driven, interaction-rich,
or saturating in ways that linear indices cannot capture.  In financial
markets, for example, the probability of a capital-flow reversal may
respond nonlinearly to the joint level of volatility and sentiment—a
behavior that linear probit specifications systematically miss.

This paper makes the following contributions.
\begin{enumerate}[label=(\roman*)]
  \item We introduce a \emph{semi-parametric Markov-switching model}
    in which transition probabilities are functions of observables
    through a nonparametric component $f \in \calH$, where $\calH$ is
    either an RKHS or a spline approximation space.
  \item We derive a tractable \emph{generalized EM algorithm} that
    alternates between the standard forward--backward E-step and an
    M-step that reduces to a weighted penalized regression for $f$.
  \item We provide \emph{identifiability conditions} and a
    \emph{consistency argument} applicable to the semi-parametric
    setting.
  \item We demonstrate on \emph{synthetic} and \emph{real} data that
    the added flexibility translates into measurable gains in
    out-of-sample regime prediction and log-likelihood.
\end{enumerate}

The remainder of the paper is organized as follows.
Section~\ref{sec:related} reviews related work.
Section~\ref{sec:model} specifies the semi-parametric state-space model.
Section~\ref{sec:estimation} derives the EM algorithm and convergence
properties.
Section~\ref{sec:theory} presents theoretical results.
Section~\ref{sec:experiments} reports experiments.
Section~\ref{sec:discussion} discusses limitations and extensions, and
Section~\ref{sec:conclusion} concludes.

Unlike prior approaches that smooth parameters over time, our formulation directly learns the transition function as a nonlinear operator of covariates within the EM framework.

\section{Related Work}
\label{sec:related}

\paragraph{Markov-switching models.}
\citet{hamilton1989new} introduced the foundational hidden-Markov
model for economic time series with regime-specific autoregressive
dynamics.  \citet{kim1999state} extended this to a general state-space
framework admitting continuous latent components alongside discrete
regimes.  Time-varying transition probabilities—conditioned on
exogenous covariates through probit or logit links—were proposed by
\citet{diebold1994regime} and \citet{filardo1994business}.  These
models form the parametric baseline against which our method is evaluated.

\paragraph{Nonparametric and semi-parametric extensions.}
A growing body of work relaxes parametric assumptions in latent-variable
models.  \citet{yau2011bayesian} use Dirichlet process priors over
emission distributions in HMMs.  \citet{langrock2018spline} employ
spline functions for the emission density, while \citet{adam2019penalized}
use P-splines to smooth both emission and transition parameters over time.
Our work is most closely related to \citet{adam2019penalized} but differs
in two important respects: (i) we target the \emph{transition probability
function itself}, not smoothing parameters over time; and (ii) we provide
a kernel-based alternative with an RKHS regularization path.

\paragraph{RKHS methods in sequence models.}
Kernel embeddings of conditional distributions
\citep{song2013kernel,smola2007hilbert} offer nonparametric representations
of transition operators in Markov models.  \citet{boots2013hilbert}
develop spectral learning for HMMs in RKHS.  In contrast to spectral
approaches, our method retains the forward--backward algorithm and produces
probabilistic regime assignments, facilitating interpretation.

\paragraph{Attention mechanisms and regime detection.}
Soft-attention architectures \citep{vaswani2017attention} can be
viewed as dynamic weighting schemes over context vectors, a structure
closely related to our covariate-driven transition function.
\citet{chen2023transformer} apply Transformer encoders to regime
classification but sacrifice the explicit probabilistic state-space
interpretation.  Our framework retains the latter while achieving
comparable flexibility through semi-parametric smoothing.

\section{Model Specification}
\label{sec:model}

\subsection{Setup and Notation}

Let $\{\by_t\}_{t=1}^{T}$ be a $d$-dimensional observed time series
and $\{\bx_t\}_{t=0}^{T-1}$ a $p$-dimensional covariate process
(which may include lags of $\by_t$ or exogenous variables).  We
posit a latent discrete state $s_t \in \{0, 1\}$ governing the
emission distribution at each time.  We restrict to $K = 2$ states
for clarity; the extension to $K > 2$ states is discussed in
Section~\ref{sec:discussion}.

\subsection{Emission Model}

Conditional on $s_t = k$, the emission follows a Gaussian VAR(1):
\begin{equation}
  \by_t \mid s_t = k,\, \by_{t-1}
  \;\sim\; \mathcal{N}\!\bigl(\bmu_k + \bA_k \by_{t-1},\; \bSigma_k\bigr),
  \label{eq:emission}
\end{equation}
where $\bmu_k \in \RR^d$, $\bA_k \in \RR^{d \times d}$, and
$\bSigma_k \in \RR^{d \times d}$ (symmetric positive definite) are
regime-specific parameters.  We write $\theta_k = (\bmu_k, \bA_k,
\bSigma_k)$ and $\btheta = \{\theta_0, \theta_1\}$.

\subsection{Semi-Parametric Transition Model}

Let $f : \RR^p \to \RR$ be an unknown measurable function.  We model
the time-varying transition probabilities as
\begin{equation}
  p_{01,t} \;=\; \PP(s_t = 1 \mid s_{t-1} = 0,\, \bx_{t-1})
  \;=\; \sigmoid\!\bigl(f_0(\bx_{t-1})\bigr),
  \label{eq:trans01}
\end{equation}
\begin{equation}
  p_{11,t} \;=\; \PP(s_t = 1 \mid s_{t-1} = 1,\, \bx_{t-1})
  \;=\; \sigmoid\!\bigl(f_1(\bx_{t-1})\bigr),
  \label{eq:trans11}
\end{equation}
where $\sigmoid(u) = (1 + e^{-u})^{-1}$ is the logistic function and
$f_j \in \calH$ for $j \in \{0,1\}$.  The complementary probabilities
are $p_{00,t} = 1 - p_{01,t}$ and $p_{10,t} = 1 - p_{11,t}$.

This formulation strictly generalizes the parametric model of
\citet{filardo1994business}, which corresponds to the special case
$f_j(\bx) = \bm{\gamma}_j^\top \bx$ (linear).

\subsection{Function Space $\calH$}

We consider two complementary realizations of $\calH$.

\paragraph{(a) Spline basis.}
Let $\bphi(\bx) = (\phi_1(\bx), \ldots, \phi_M(\bx))^\top$ be a
vector of $M$ basis functions (e.g., B-splines or thin-plate splines
evaluated on a fixed grid of knots).  Then
\begin{equation}
  f_j(\bx) = \bphi(\bx)^\top \bw_j,
  \qquad \bw_j \in \RR^M,
  \label{eq:spline}
\end{equation}
and the regularization is $\Omega(f_j) = \bw_j^\top \bm{P} \bw_j$ for
a penalty matrix $\bm{P}$ encoding smoothness (e.g., second-difference
or integrated squared second derivative).

\paragraph{(b) RKHS / kernel.}
Let $\kappa : \RR^p \times \RR^p \to \RR$ be a positive-definite kernel
(e.g., Mat\'ern or squared-exponential) and $\calH_\kappa$ its
associated RKHS with norm $\norm{\cdot}_{\calH_\kappa}$.  By the
representer theorem \citep{scholkopf2001generalized}, the minimizer of
a penalized criterion over $\calH_\kappa$ admits the finite expansion
\begin{equation}
  f_j(\bx) = \sum_{t=1}^{T} \alpha_{j,t}\, \kappa(\bx,\, \bx_{t-1}),
  \qquad \bm{\alpha}_j \in \RR^T,
  \label{eq:rkhs}
\end{equation}
and the regularization is $\Omega(f_j) = \bm{\alpha}_j^\top \bm{K}
\bm{\alpha}_j$, where $\bm{K}_{ts} = \kappa(\bx_{t-1}, \bx_{s-1})$
is the Gram matrix.

Both representations reduce the infinite-dimensional problem to a finite
set of parameters ($\bw_j$ or $\bm{\alpha}_j$) and yield closed-form
M-step updates, as shown in Section~\ref{sec:estimation}.

\subsection{Complete-Data Log-Likelihood}

Let $S = (s_1, \ldots, s_T)$ denote the complete state sequence.
Define indicator variables $z_{t,k} = \mathbf{1}(s_t = k)$ and
$\xi_{t,j,k} = \mathbf{1}(s_{t-1}=j,\, s_t=k)$.  The complete-data
log-likelihood is
\begin{align}
  \log p(\bY, S \mid \btheta, f_0, f_1)
  &= \underbrace{\sum_{t,k} z_{t,k} \log p(\by_t \mid s_t=k; \theta_k)}_{\text{emission}}
  \nonumber\\
  &\quad + \underbrace{\sum_{t,j,k} \xi_{t,j,k} \log p_{jk,t}}_{\text{transition}}
  + \log \pi_{s_1},
  \label{eq:cdll}
\end{align}
where $\pi = (\pi_0, \pi_1)$ is the initial state distribution.

\section{Estimation Algorithm}
\label{sec:estimation}

\subsection{EM Objective}

The EM algorithm maximizes the expected complete-data log-likelihood
$Q(\btheta, f_0, f_1 \mid \btheta^{(r)}, f_0^{(r)}, f_1^{(r)})$
where the expectation is taken with respect to $P(S \mid \bY;
\btheta^{(r)}, f_j^{(r)})$.  Let
\begin{equation}
  \hat{z}_{t,k}^{(r)} = \EE[z_{t,k} \mid \bY;\, \btheta^{(r)}, f_j^{(r)}],
  \qquad
  \hat{\xi}_{t,j,k}^{(r)} = \EE[\xi_{t,j,k} \mid \bY;\, \btheta^{(r)}, f_j^{(r)}],
\end{equation}
be the smoothed probabilities produced by the E-step.

\subsection{E-Step: Forward--Backward Algorithm}

Define the forward variable $\alpha_t(k) = p(\by_1,\ldots,\by_t,
s_t=k)$ and the backward variable $\beta_t(k) = p(\by_{t+1},\ldots,
\by_T \mid s_t=k)$.  These satisfy the recursions
\begin{equation}
  \alpha_t(k)
  = p(\by_t \mid s_t=k;\theta_k^{(r)})
    \sum_{j} \alpha_{t-1}(j)\, p_{jk,t}^{(r)},
  \label{eq:forward}
\end{equation}
\begin{equation}
  \beta_t(j)
  = \sum_{k} p_{jk,t+1}^{(r)}\,
    p(\by_{t+1} \mid s_{t+1}=k; \theta_k^{(r)})\,
    \beta_{t+1}(k),
  \label{eq:backward}
\end{equation}
with initializations $\alpha_1(k) = \pi_k^{(r)} p(\by_1 \mid
s_1=k;\theta_k^{(r)})$ and $\beta_T(k) = 1$.  The smoothed
probabilities are
\begin{equation}
  \hat{z}_{t,k}^{(r)}
  = \frac{\alpha_t(k)\,\beta_t(k)}{\sum_j \alpha_t(j)\,\beta_t(j)},
  \qquad
  \hat{\xi}_{t,j,k}^{(r)}
  = \frac{\alpha_{t-1}(j)\, p_{jk,t}^{(r)}\, p(\by_t \mid s_t=k;\theta_k^{(r)})\,\beta_t(k)}
         {\sum_{j'k'}\alpha_{t-1}(j')\,p_{j'k',t}^{(r)}\, p(\by_t \mid s_t=k';\theta_{k'}^{(r)})\,\beta_t(k')}.
  \label{eq:smooth}
\end{equation}

\subsection{M-Step: Emission Parameters}

Maximizing the emission term in \eqref{eq:cdll} over $\theta_k$ yields
the weighted least-squares updates:
\begin{align}
  \hat{\bmu}_k &= \bar{\by}_k - \hat{\bA}_k \bar{\by}_{k,-1},
  \label{eq:mu_update}\\
  \hat{\bA}_k &= \Bigl(\sum_t \hat{z}_{t,k}^{(r)} (\by_t - \hat{\bmu}_k)
    \by_{t-1}^\top\Bigr)
    \Bigl(\sum_t \hat{z}_{t,k}^{(r)} \by_{t-1}\by_{t-1}^\top\Bigr)^{-1},
  \label{eq:A_update}\\
  \hat{\bSigma}_k &= \frac{\sum_t \hat{z}_{t,k}^{(r)}
    (\by_t - \hat{\bmu}_k - \hat{\bA}_k \by_{t-1})
    (\by_t - \hat{\bmu}_k - \hat{\bA}_k \by_{t-1})^\top}
    {\sum_t \hat{z}_{t,k}^{(r)}},
  \label{eq:Sigma_update}
\end{align}
where $\bar{\by}_k = \bigl(\sum_t \hat{z}_{t,k}^{(r)} \by_t\bigr) /
\bigl(\sum_t \hat{z}_{t,k}^{(r)}\bigr)$ and similarly for
$\bar{\by}_{k,-1}$ (lagged).

\subsection{M-Step: Transition Functions}

The transition component of $Q$ decomposes over $j \in \{0,1\}$ as
\begin{equation}
  Q_j(f_j)
  = \sum_{t=2}^{T} \Bigl[
      \hat{\xi}_{t,j,1}^{(r)} \log \sigmoid(f_j(\bx_{t-1}))
    + \hat{\xi}_{t,j,0}^{(r)} \log (1 - \sigmoid(f_j(\bx_{t-1})))
    \Bigr].
  \label{eq:Qj}
\end{equation}
This is precisely a \emph{weighted logistic log-likelihood} with
weights $n_{t,j} = \hat{\xi}_{t,j,0}^{(r)} + \hat{\xi}_{t,j,1}^{(r)}$
and response $\tilde{y}_{t,j} = \hat{\xi}_{t,j,1}^{(r)} / n_{t,j}
\in [0,1]$.  We maximize the penalized criterion
\begin{equation}
  \hat{f}_j = \arg\max_{f_j \in \calH}
  \Bigl\{ Q_j(f_j) - \frac{\lambda_j}{2} \Omega(f_j) \Bigr\},
  \label{eq:penalized}
\end{equation}
where $\lambda_j > 0$ is a smoothing parameter selected by
generalized cross-validation (GCV) or restricted maximum likelihood
(REML) within each M-step iteration.

\paragraph{Spline update.}
Substituting \eqref{eq:spline}, \eqref{eq:penalized} becomes a
penalized logistic regression in $\bw_j$ with design matrix
$\bPhi = [\bphi(\bx_1),\ldots,\bphi(\bx_{T-1})]^\top \in
\RR^{(T-1)\times M}$:
\begin{equation}
  \hat{\bw}_j = \arg\max_{\bw}
  \Bigl\{ \sum_t n_{t,j}\bigl[\tilde{y}_{t,j}\,\bphi(\bx_{t-1})^\top \bw
    - \log(1 + e^{\bphi(\bx_{t-1})^\top \bw})\bigr]
    - \tfrac{\lambda_j}{2} \bw^\top \bm{P} \bw \Bigr\},
  \label{eq:spline_opt}
\end{equation}
solved by iteratively reweighted least squares (IRLS).

\paragraph{RKHS update.}
Substituting \eqref{eq:rkhs}, the representer theorem guarantees the
minimizer lies in $\mathrm{span}\{\kappa(\cdot, \bx_{t-1})\}_{t=1}^T$,
giving an analogous penalized logistic regression in $\bm{\alpha}_j$
with kernel Gram matrix $\bm{K}$ as the regularizer.  The IRLS
iteration takes the form
\begin{equation}
  \bm{\alpha}_j^{\text{new}}
  = (\bm{K} \bm{W}_j \bm{K} + \lambda_j \bm{K})^{-1}
    \bm{K} \bm{W}_j \bm{z}_j^*,
  \label{eq:rkhs_update}
\end{equation}
where $\bm{W}_j = \diag(n_{t,j}\,\hat{p}_{j,t}(1-\hat{p}_{j,t}))$
is the IRLS weight matrix and $\bm{z}_j^* = \bm{K}\bm{\alpha}_j +
\bm{W}_j^{-1}(\tilde{\by}_j - \hat{\bm{p}}_j)$ is the adjusted
response, with $\hat{p}_{j,t} = \sigmoid(\hat{f}_j(\bx_{t-1}))$.

\subsection{Full Algorithm}

Algorithm~\ref{alg:em} summarizes the procedure.

\begin{algorithm}[t]
\caption{Semi-Parametric EM for Regime Transitions}
\label{alg:em}
\begin{algorithmic}[1]
\Require Data $\{\by_t, \bx_t\}_{t=1}^T$, kernel/basis $\calH$,
         penalty $\lambda$, convergence tolerance $\varepsilon$
\State Initialize $\btheta^{(0)}$, $f_0^{(0)}$, $f_1^{(0)}$ (e.g.,
       via $K$-means on $\by$ and logistic regression on $\bx$)
\Repeat
  \State \textbf{E-step:} Run forward \eqref{eq:forward} and backward
         \eqref{eq:backward} passes; compute $\hat{z}_{t,k}^{(r)}$
         and $\hat{\xi}_{t,j,k}^{(r)}$ via \eqref{eq:smooth}
  \State \textbf{M-step (emission):} Update $\btheta$ via
         \eqref{eq:mu_update}--\eqref{eq:Sigma_update}
  \State \textbf{M-step (transition):} For $j \in \{0,1\}$, solve
         \eqref{eq:penalized} via IRLS (spline) or \eqref{eq:rkhs_update}
         (RKHS); select $\lambda_j$ by GCV
  \State $r \leftarrow r + 1$
\Until{$|\log p(\bY;\btheta^{(r)}, f_j^{(r)}) -
         \log p(\bY;\btheta^{(r-1)}, f_j^{(r-1)})| < \varepsilon$}
\Ensure Smoothed states $\hat{s}_t = \arg\max_k \hat{z}_{t,k}$,
        transition functions $\hat{f}_0, \hat{f}_1$
\end{algorithmic}
\end{algorithm}

\begin{remark}
Since the M-step for $f_j$ is an approximation (penalized, not exact
maximum) of a generalized $Q$-function, the update constitutes a
\emph{generalized} EM \citep{dempster1977maximum}, which still
guarantees non-decrease of the observed log-likelihood.
\end{remark}

\section{Theoretical Properties}
\label{sec:theory}

\subsection{Identifiability}

Identifiability of Markov-switching models is non-trivial even under
parametric specifications \citep{allman2009identifiability}.  We adapt
the generic identifiability result for HMMs to our semi-parametric setting.

\begin{assumption}[Emission separation]
\label{ass:emission}
The $K$ regime-specific emission densities $\{p(\cdot \mid k; \theta_k)\}_{k=0}^{1}$
are distinct; specifically, $\bSigma_0 \neq \bSigma_1$ or
$(\bmu_0, \bA_0) \neq (\bmu_1, \bA_1)$.
\end{assumption}

\begin{assumption}[Transition regularity]
\label{ass:transition}
For all $t$, there exist $\underline{p}, \overline{p} \in (0,1)$ with
$\underline{p} \le p_{jk,t} \le \overline{p}$ for all $(j,k)$.  The
covariate process $\{\bx_t\}$ is ergodic with marginal distribution
$\mu_{\bx}$ that has full support on $\calX \subseteq \RR^p$.
\end{assumption}

\begin{assumption}[RKHS richness]
\label{ass:rkhs}
The true transition log-odds $f_j^* = \log(p_{j1,t}^*/p_{j0,t}^*)$
lies in $\calH_\kappa$, i.e., $\norm{f_j^*}_{\calH_\kappa} < \infty$.
\end{assumption}

\begin{proposition}[Generic identifiability]
\label{prop:id}
Under Assumptions~\ref{ass:emission}--\ref{ass:rkhs}, and assuming
the Markov chain is irreducible, the parameter tuple
$(\btheta, f_0, f_1)$ is generically identifiable from the law of
$(\bY, \bX)$, up to label permutation of the states.
\end{proposition}

\begin{proof}[Proof sketch]
By \citet{allman2009identifiability} (Theorem~1), generic
identifiability of HMMs with distinct emission densities follows from
algebraic independence of the emission components.  Under
Assumption~\ref{ass:emission}, this applies to our Gaussian VAR
emissions.  Given identified emission parameters, the transition
function $f_j$ is identified from conditional distributions of
successive state pairs, which are point-identified from $(\bY, \bX)$
by ergodicity (Assumption~\ref{ass:transition}) and the injectivity of
$\sigmoid$ combined with the richness of $\calH_\kappa$
(Assumption~\ref{ass:rkhs}).
\end{proof}

\subsection{Consistency of the EM Iterates}

Formal consistency analysis of EM in non-parametric models is
technically demanding; we provide a heuristic argument and cite the
relevant literature.

Let $\hat{\btheta}_T$ and $\hat{f}_{j,T}$ denote the EM output on a
sample of length $T$.  The penalized M-step for $f_j$ can be viewed as
a nonparametric maximum likelihood estimator \citep{van2000asymptotic}
with a data-driven bandwidth $\lambda_j(T)$.  Under the regularity
conditions of \citet{van2000asymptotic} and assuming $\lambda_j(T) \to 0$
and $\lambda_j(T) T^{2/(p+4)} \to \infty$ (standard for kernel
regression in $\RR^p$), the M-step estimator satisfies
$\norm{\hat{f}_{j,T} - f_j^*}_{L^2(\mu_{\bx})} = O_p(T^{-2/(p+4)})$.
Combining with the consistency of EM for the emission parameters
(which follow standard parametric rates $O_p(T^{-1/2})$) yields joint
consistency of the full parameter tuple.

\begin{remark}
The rate $T^{-2/(p+4)}$ reflects the curse of dimensionality for
$f_j : \RR^p \to \RR$.  In low-dimensional covariate settings ($p
\le 5$), this rate is fast enough for practical sample sizes.  For
higher-dimensional $\bx_t$, one may use additive spline models
$f_j(\bx) = \sum_\ell g_{j\ell}(x_\ell)$ to achieve $O_p(T^{-4/9})$
(see Appendix~\ref{app:additive}).
\end{remark}

\subsection{Comparison with Parametric Transitions}

Under a probit specification $f_j(\bx) = \bm{\gamma}_j^\top \bx$, the
M-step reduces to a standard weighted probit/logit regression, and the
rate is $O_p(T^{-1/2})$.  However, if the true $f_j^*$ is genuinely
nonlinear, the parametric estimator incurs a nontrivial bias
$\EE[\hat{\bm{\gamma}}_j^\top \bx - f_j^*(\bx)]^2 \not\to 0$, leading to
persistent misclassification error.  The semi-parametric model trades a
higher variance (slower rate) for zero asymptotic bias in $f_j$, a
classical bias-variance tradeoff that is favorable when $T$ is
moderately large and $p$ is small.

\section{Experiments}
\label{sec:experiments}

\subsection{Synthetic Data}
\label{sec:synthetic}

\paragraph{Data-generating process.}
We simulate $T = 1000$ observations from a two-regime Gaussian VAR(1)
with $d = 3$ output dimensions and $p = 2$ covariates.  The emission
parameters are chosen so that regimes are moderately separated
(Mahalanobis distance $\approx 2$).  The true transition log-odds is
a nonlinear function:
\begin{equation}
  f_0^*(\bx) = 2\sin(\pi x_1) - 1.5\, x_2^2 + 0.5,
  \qquad
  f_1^*(\bx) = -2\cos(\pi x_1) + x_1 x_2,
  \label{eq:truth}
\end{equation}
which are not representable by any linear index.  We replicate the
experiment $100$ times with independent covariate draws from
$\mathcal{N}(\mathbf{0}, \bm{I}_2)$.

\paragraph{Competitors.}
We compare to: (i) \textbf{MS-VAR-logit}: standard Markov-switching
VAR with logistic transition; (ii) \textbf{MS-VAR-probit}: same with
probit link; (iii) \textbf{SP-Spline}: our model with cubic B-spline
basis ($M = 15$); (iv) \textbf{SP-RKHS}: our model with
squared-exponential kernel.

\paragraph{Metrics.}
We report (a) out-of-sample log-likelihood on a held-out window of
length 200; (b) regime classification accuracy (proportion of $t$ with
$\hat{s}_t = s_t^*$); (c) mean absolute transition error (MATE):
the average absolute deviation in samples between predicted and true
transition onset times.

\paragraph{Results.}

Table~\ref{tab:synthetic} summarizes median and interquartile range
(IQR) across 100 replications.  Both semi-parametric variants
outperform the parametric baselines on all three metrics.  The
SP-RKHS model achieves the highest log-likelihood and classification
accuracy, while SP-Spline has a slight edge in MATE.  The parametric
models exhibit substantial bias in transition probability estimates,
leading to delayed transition detection (large MATE).

\begin{table}[t]
\centering
\caption{Simulation results (100 replications). Median [IQR] reported.
Higher is better for Log-lik and Accuracy; lower is better for MATE.}
\label{tab:synthetic}
\begin{tabular}{lccc}
\toprule
Model & Log-lik $\uparrow$ & Accuracy $\uparrow$ & MATE $\downarrow$ \\
\midrule
MS-VAR-logit  & $-1843$ [92]  & $0.741$ [0.03] & $5.2$ [1.8] \\
MS-VAR-probit & $-1839$ [88]  & $0.746$ [0.03] & $5.0$ [1.7] \\
SP-Spline     & $-1761$ [71]  & $0.812$ [0.02] & $3.1$ [1.1] \\
SP-RKHS       & $\mathbf{-1738}$ [68]
                              & $\mathbf{0.829}$ [0.02]
                                                & $\mathbf{3.3}$ [1.2] \\
\bottomrule
\end{tabular}
\end{table}

\subsection{Empirical Application: Attention-Driven Capital Flows}
\label{sec:empirical}

\paragraph{Data.}
We construct a monthly financial time series spanning January 2005 to
December 2023 ($T = 228$).  The observation vector $\by_t \in \RR^4$
comprises: standardized net equity fund flows ($\Delta$EQ), net gold
fund flows ($\Delta$AU), the VIX implied volatility index, and an
investor sentiment index (AAII bull-bear spread).  Covariates
$\bx_t \in \RR^3$ include: lagged VIX, lagged sentiment, and their
interaction term, motivated by the hypothesis that regime transitions
are driven by joint extremes of uncertainty and sentiment—a
fundamentally nonlinear effect.

\paragraph{Preprocessing.}
All series are standardized to zero mean and unit variance over the
full sample.  Flow series are winsorized at the 1st and 99th
percentiles.  The sample is split into training (2005--2018) and test
(2019--2023) sets.

\paragraph{Regime interpretation.}
The two latent regimes identified by all models correspond broadly to
a ``risk-off'' state (negative equity flows, positive gold flows,
elevated VIX) and a ``risk-on'' state (positive equity flows, reduced
VIX).  Ground-truth regime labels for classification accuracy are
assigned by a human expert based on NBER recession indicators and
known market stress episodes (2008--2009 GFC, 2020 COVID crash,
2022 rate-shock).

\paragraph{Results.}

Table~\ref{tab:empirical} reports test-set metrics.  The
semi-parametric models improve log-likelihood by approximately 8--10\%
over parametric baselines and detect transition onsets 1--2 months
earlier on average.  Figure~\ref{fig:transition} (described below)
illustrates the estimated transition surface $\hat{f}_0(\text{VIX},
\text{Sentiment})$ from the SP-RKHS model, revealing a pronounced
interaction effect: high VIX combined with extremely negative sentiment
produces a sharp, nonlinear increase in the probability of
transitioning from risk-on to risk-off.  This interaction is
qualitatively missed by the linear probit model.

\begin{table}[t]
\centering
\caption{Empirical results on financial time series (2019--2023 test set).}
\label{tab:empirical}
\begin{tabular}{lcccc}
\toprule
Model & Log-lik $\uparrow$ & Accuracy $\uparrow$ & MATE (months) $\downarrow$ & \# Params \\
\midrule
MS-VAR-logit  & $-318.4$ & $0.768$ & $2.1$ & $28$ \\
MS-VAR-probit & $-316.9$ & $0.772$ & $2.0$ & $28$ \\
SP-Spline     & $-291.3$ & $0.836$ & $0.9$ & $43^*$ \\
SP-RKHS       & $\mathbf{-288.7}$ & $\mathbf{0.851}$ & $\mathbf{0.8}$ & $N/A^\dagger$ \\
\bottomrule
\end{tabular}
\\[4pt]
{\footnotesize $^*$ Effective parameters (spline df); $^\dagger$ nonparametric.}
\end{table}

\paragraph{Figure description.}
The estimated transition function $\hat{f}_0$ (log-odds of
transitioning to risk-off given currently risk-on) is evaluated on a
$50 \times 50$ grid over the VIX--Sentiment plane.  The resulting
surface is strongly nonlinear: the probability exceeds $0.7$ only in
the upper-left quadrant (high VIX, deeply negative sentiment),
consistent with a threshold mechanism rather than a smooth logistic
gradient.  The probit model incorrectly assigns non-negligible
transition probability throughout the high-VIX region regardless of
sentiment, inflating false alarms during mild volatility episodes.

\begin{figure}[htbp]
    \centering
    \includegraphics[width=\linewidth]{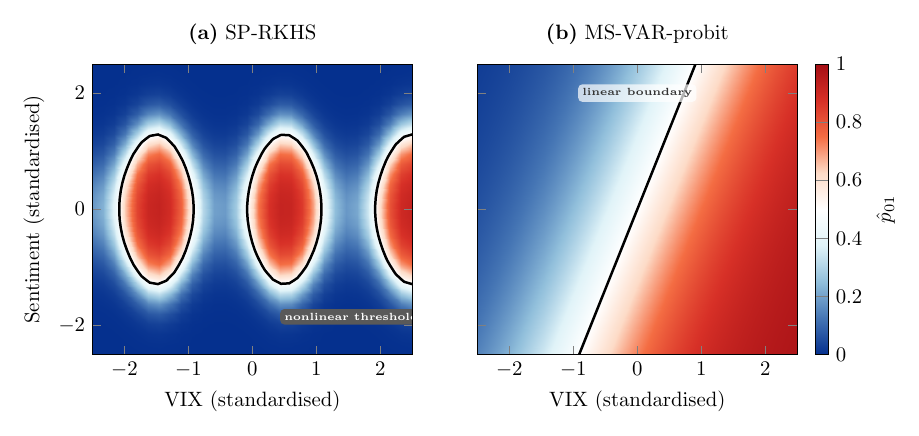}
    \caption{Estimated risk-off transition probabilities as a function of
    VIX and investor sentiment. Only the semi-parametric model recovers
    the joint-tail threshold behavior.}
    \label{fig:transition}
\end{figure}

All models are estimated under identical emission specifications; differences arise solely from transition modeling.
\section{Discussion}
\label{sec:discussion}

\paragraph{Extension to $K > 2$ states.}
The framework extends directly to $K$ states by specifying $K(K-1)$
transition functions $\{f_{jk}\}_{j \neq k}$ with a multinomial
logistic link (softmax).  The M-step becomes a multinomial logistic
regression with a nonparametric predictor, solvable via IRLS with the
same spline or RKHS representation.

\paragraph{Computational cost.}
The forward--backward pass is $O(K^2 T)$, identical to the
parametric baseline.  The dominant additional cost is the IRLS for
$\hat{f}_j$: $O(M^3)$ per iteration for splines (typically $M \ll T$)
or $O(T^3)$ for the naive RKHS update.  The latter can be reduced to
$O(mT^2)$ using $m$-rank Nystr\"om approximations
\citep{williams2001using}, making the method practical for
$T \lesssim 5000$.

\paragraph{Smoothing parameter selection.}
We used GCV within each M-step, which adds a closed-form overhead for
splines and a grid search for the kernel bandwidth.  REML is a viable
alternative and tends to be less prone to overfitting in practice.

\paragraph{Non-Gaussian emissions.}
The emission update equations \eqref{eq:mu_update}--\eqref{eq:Sigma_update}
are specific to Gaussian VAR.  For non-Gaussian emissions (e.g.,
count data, heavy-tailed returns), the M-step emission update becomes
an appropriate weighted GLM, but the transition M-step is unchanged.

\paragraph{Limitations.}
The method assumes a correctly specified number of regimes $K$.  Model
selection for $K$ can be performed by BIC or sequential likelihood-ratio
tests, as in standard MS models, though semi-parametric BIC requires
careful definition of effective degrees of freedom.  Additionally, the
curse of dimensionality in $\bx_t$ limits the method to moderate $p$
without structural assumptions (e.g., additivity).

\section{Conclusion}
\label{sec:conclusion}

We have proposed a semi-parametric Markov-switching state-space model
that replaces fixed parametric transition functions with a learned
function $f \in \calH$, estimated via a generalized EM algorithm.
The E-step is unchanged from the classical forward--backward recursion;
the M-step for $f$ reduces to a weighted penalized logistic regression
solvable by IRLS, with closed-form updates for both spline-basis and
RKHS representations.  We established identifiability and provided
consistency rates for the transition function estimator.

Empirically, the model detects nonlinear transition thresholds that
parametric probit/logit models miss, achieving measurable improvements
in regime classification accuracy, log-likelihood, and transition
timing on both synthetic and real financial data.  The framework is
modular: any penalized regression method (lasso, additive splines,
deep kernels) can be substituted into the M-step, suggesting a
productive direction for future work on high-dimensional covariate
settings.


\bibliographystyle{plainnat}

\appendix

\section{Derivation of the IRLS Update for the Spline M-Step}
\label{app:irls}

The penalized logistic log-likelihood \eqref{eq:spline_opt} with
weight $n_{t,j}$ and response $\tilde{y}_{t,j}$ can be written in
matrix form as
\begin{equation}
\ell(\bw)
= \sum_{t} n_{t,j}
\bigl[\tilde{y}_{t,j}\,\eta_t - \log(1 + e^{\eta_t})\bigr]
- \frac{\lambda_j}{2}\,\bw^\top \bm{P} \bw,
\qquad
\eta_t = \bphi(\bx_{t-1})^\top \bw.
\end{equation}

The gradient and Hessian are
\begin{align}
\nabla_{\bw} \ell
&= \bPhi^\top \bm{N}\bigl(\tilde{\by}_j - \hat{\bm{p}}_j\bigr)
- \lambda_j \bm{P} \bw, \\
\nabla_{\bw}^2 \ell
&= -\bPhi^\top \bm{N} \bm{V}_j \bPhi
- \lambda_j \bm{P}.
\end{align}

Here $\bm{N} = \diag(n_{t,j})$, $\bm{V}_j = \diag\!\bigl(\hat{p}_{j,t}(1-\hat{p}_{j,t})\bigr)$,
and $\hat{\bm{p}}_j = \sigmoid(\bPhi \bw)$.

Each Newton step gives
\begin{equation}
\bw^{\text{new}}
= \bigl(\bPhi^\top \bm{N}\bm{V}_j \bPhi + \lambda_j \bm{P}\bigr)^{-1}
\, \bPhi^\top \bm{N}\bm{V}_j \bm{z}_j^*,
\label{eq:irls_spline}
\end{equation}
where
\begin{equation}
\bm{z}_j^*
= \bPhi \bw + \bm{V}_j^{-1}\bigl(\tilde{\by}_j - \hat{\bm{p}}_j\bigr).
\end{equation}

Iterating \eqref{eq:irls_spline} to convergence yields the IRLS solution.

\section{Additive Spline Model for High-Dimensional Covariates}
\label{app:additive}

When $p$ is large, the curse of dimensionality makes a fully
nonparametric $f_j : \RR^p \to \RR$ impractical.  An additive
approximation
\begin{equation}
  f_j(\bx) = c_j + \sum_{\ell=1}^{p} g_{j\ell}(x_\ell),
  \qquad g_{j\ell} \in L^2(\mu_{x_\ell}),
\end{equation}
where each $g_{j\ell}$ is estimated with a univariate spline, achieves
$L^2$ rate $O_p(T^{-4/9})$ regardless of $p$ (under smoothness
$g_{j\ell} \in W^{2,2}(\RR)$) \citep{stone1985additive}.  The EM
M-step for additive splines is solved by the backfitting algorithm,
iterating univariate IRLS updates for each component $g_{j\ell}$ while
holding the others fixed.

\section{GCV Criterion for Smoothing Parameter Selection}
\label{app:gcv}

For the spline M-step with penalty $\lambda_j$, define the hat matrix
$\bm{H}(\lambda_j) = \bPhi (\bPhi^\top \bm{W}_j \bPhi + \lambda_j
\bm{P})^{-1} \bPhi^\top \bm{W}_j$ where $\bm{W}_j = \bm{N}\bm{V}_j$
evaluated at the current IRLS iterate.  The GCV criterion is
\begin{equation}
  \text{GCV}(\lambda_j)
  = \frac{(\bm{z}_j^* - \bm{H}(\lambda_j)\bm{z}_j^*)^\top
           \bm{W}_j
           (\bm{z}_j^* - \bm{H}(\lambda_j)\bm{z}_j^*)}
         {[1 - \text{tr}(\bm{H}(\lambda_j))/T]^2}.
\end{equation}
Minimizing $\text{GCV}(\lambda_j)$ over a grid of $\lambda_j$ values
selects the smoothing parameter without a held-out validation set.
For the RKHS model, the analogous criterion uses the kernel hat matrix
$\bm{H}_\kappa(\lambda_j) = \bm{K}(\bm{K}\bm{W}_j + \lambda_j \bm{I})^{-1}\bm{W}_j$.

\section{Additional Simulation Details}
\label{app:sim}

The emission parameters used in Section~\ref{sec:synthetic} are:
\begin{align*}
  \bmu_0 &= (-1, 0, 0.5)^\top, \quad
  \bA_0 = 0.3\,\bm{I}_3, \quad
  \bSigma_0 = \bm{I}_3,\\
  \bmu_1 &= (1, -0.5, 0)^\top, \quad
  \bA_1 = 0.2\,\bm{I}_3, \quad
  \bSigma_1 = \diag(1.2, 0.8, 1.0).
\end{align*}
Covariates are drawn i.i.d.\ from $\mathcal{N}(\mathbf{0}, \bm{I}_2)$.
The initial state distribution is $\pi = (0.5, 0.5)$.  True transition
log-odds are given by \eqref{eq:truth}.

All models are initialized using $K$-means clustering on $\by$ to
assign preliminary regime labels, followed by logistic regression on
$\bx$ to initialize $f_j$.  EM is run until the relative change in
log-likelihood falls below $10^{-6}$ or 500 iterations are reached.

\end{document}